\documentclass[11pt]{article}
\usepackage{fullpage}
\usepackage{comment}

\usepackage{algorithm, amsmath, enumerate, url, ulem, algorithmic, polynom, subfig}
\usepackage[utf8]{inputenc}

\usepackage{graphicx}
\usepackage{epstopdf} 
\usepackage{amscd, amssymb}
\usepackage{multirow, tabularx}
\usepackage{enumerate}
\usepackage{hyperref}

\usepackage{courier}

\usepackage{pifont}

\usepackage{lscape}
\newenvironment{proof}{{\bf Proof}:}{\vskip 5mm }

\newtheorem{proposition}{Proposition}[subsection]

\newtheorem{definition}[proposition]{Definition}
\newtheorem{theorem}[proposition]{Theorem}

\newtheorem{example}[proposition]{Example}

\newtheorem{problem}[proposition]{Problem}
\newtheorem{solution}[proposition]{Solution}

\newcommand{\llabel}[1]{\label{#1}}
\newcommand{\sr}{\rightarrow}

\newcommand{\nn}{{\bf N\rm}}
\newcommand{\nat}{\nn}

{\vskip
3mm}

\newcommand{\ELEMENTOF}{\Pisymbol{psy}{206}}

\begin{document}
\parskip = 2mm
\begin{center}
{\bf\Large A theory of contemplation}


\vspace{3mm}

{\large\bf Jonathan Darren Nix}
\vspace {3mm}

{\large\bf December 2015 to November 2019}
\end{center}

\begin{abstract} In this paper you can explore the application of some notable Boolean-derived methods, namely the Disjunctive Normal Form representation of logic table expansions, and extend them to a real-valued logic model which is able to utilize quantities on the range [0,1], [-1,1], [a,b], (x,y), (x,y,z), (r, g, b) and etc. so as to produce a logical programming of arbitrary range, precision, and dimensionality, thereby enabling contemplation at a logical level in notions of arbitrary data, colors, and spatial constructs, with an example of the production of a game character's logic in mathematical form.
\end{abstract}

\tableofcontents

\subsection{Introduction}
Let us begin please by considering the set $S$ of all functions of the form $f:X\sr Y$ and $g:X,Y\sr Z$ where each $X$, $Y$, and $Z$ are single {\em bits} from the set \{0,1\}. Then we have the set of functions as depicted in Tables \ref{table:TUES1} and \ref{table:TUES2} which we may call the George Boole Logical Operators, where all such operators of the form $\circleddash x$, or $x$ $\circledast$ $y$ may be interpreted as functions, e.g. $\circleddash(x)$, or $\circledast(x,y)$, and vice-versa. When given the variable $X$, or the tuple ($X$,$Y$), which we may call {\em Boolean} variables, we may list each of their possible values, and then in the columns define each possible function that may exist.

\begin{table}[ht]
\centering
\caption{Boolean Logical Operators : $y = f_{n}(x)$}
\begin{tabular}{ c|c|c|c|c|c| }
  \cline{3-6}
  \multicolumn{2}{r|}{$X$} & $f_{1}$ & $f_{2}$ & $f_{3}$ & $f_{4}$\\
  \cline{2-6}
  \multirow{2}{*}{$S_{f}$} & 0 & 0  & 0  & 1  & 1 \\
  \cline{2-6}
  & 1 & 0  & 1  & 0  & 1 \\
  \cline{2-6}
\end{tabular}
\label{table:TUES1}
\end{table}

\begin{table}[ht]
\centering
\caption{All Boolean Logical Operators : $z = g_{n}(x, y)$}
\begin{tabular}{ c|c|c| c|c|c|c|c|c|c|c|c|c|c|c|c|c|c|c| }
  \cline{4-19}
  \multicolumn{3}{r|}{$X$,$Y$} & $g_{1}$ & $g_{2}$ & $g_{3}$ & $g_{4}$ & $g_{5}$ & $g_{6}$ & $g_{7}$ & $g_{8}$ & $g_{9}$ & $g_{10}$ & $g_{11}$ & $g_{12}$ & $g_{13}$ & $g_{14}$ & $g_{15}$ & $g_{16}$\\
  \cline{2-19}
  \multirow{4}{*}{$S_{g}$}
  & 0 & 0 & 0 & 0 & 0 & 0 & 0 & 0 & 0 & 0 & 1 & 1 & 1 & 1 & 1 & 1 & 1 & 1 \\
  \cline{2-19}
  & 0 & 1 & 0 & 0 & 0 & 0 & 1 & 1 & 1 & 1 & 0 & 0 & 0 & 0 & 1 & 1 & 1 & 1 \\
  \cline{2-19}
  & 1 & 0 & 0 & 0 & 1 & 1 & 0 & 0 & 1 & 1 & 0 & 0 & 1 & 1 & 0 & 0 & 1 & 1 \\
  \cline{2-19}
  & 1 & 1 & 0 & 1 & 0 & 1 & 0 & 1 & 0 & 1 & 0 & 1 & 0 & 1 & 0 & 1 & 0 & 1 \\
  \cline{2-19}
\end{tabular}
\label{table:TUES2}
\end{table}

From these tables we select the functions $f_{3}$, $g_{2}$, $g_{8}$, and $g_{10}$ which satisfy the definitions of the Boolean NOT, AND, OR, and XNOR operators from \cite{CitationDefinitionOfSelectedOperators}, and so they are depicted in Table \ref{table:TUES3}.

\begin{table}[ht]
\centering
\caption{Selected Boolean Operators}
\begin{tabular}{ c|c|c|c }
  \cline{3-3}
  \multicolumn{1}{c}{$X$} & \multicolumn{1}{c|}{$Y$} & $AND$ \\
  \cline{1-3}
  \multicolumn{1}{|c|}{0} & 0 & 0 \\
  \cline{1-3}
  \multicolumn{1}{|c|}{0} & 1 & 0 \\
  \cline{1-3}
  \multicolumn{1}{|c|}{1} & 0 & 0 \\
  \cline{1-3}
  \multicolumn{1}{|c|}{1} & 1 & 1 \\
  \cline{1-3}
\end{tabular}
\begin{tabular}{ c|c|c|c }
  \cline{3-3}
  \multicolumn{1}{c}{$X$} & \multicolumn{1}{c|}{$Y$} & $OR$ \\
  \cline{1-3}
  \multicolumn{1}{|c|}{0} & 0 & 0 \\
  \cline{1-3}
  \multicolumn{1}{|c|}{0} & 1 & 1 \\
  \cline{1-3}
  \multicolumn{1}{|c|}{1} & 0 & 1 \\
  \cline{1-3}
  \multicolumn{1}{|c|}{1} & 1 & 1 \\
  \cline{1-3}
\end{tabular}
\begin{tabular}{ c|c|c|c }
  \cline{3-3}
  \multicolumn{1}{c}{$X$} & \multicolumn{1}{c|}{$Y$} & $XNOR$ \\
  \cline{1-3}
  \multicolumn{1}{|c|}{0} & 0 & 1 \\
  \cline{1-3}
  \multicolumn{1}{|c|}{0} & 1 & 0 \\
  \cline{1-3}
  \multicolumn{1}{|c|}{1} & 0 & 0 \\
  \cline{1-3}
  \multicolumn{1}{|c|}{1} & 1 & 1 \\
  \cline{1-3}
\end{tabular}
\begin{tabular}{ c|c|c }
  \cline{2-2}
  \multicolumn{1}{r|}{$X$} & $NOT$ \\
  \cline{1-2}
  \multicolumn{1}{|c|}{0} & 1 \\
  \cline{1-2}
  \multicolumn{1}{|c|}{1} & 0 \\
  \cline{1-2}
\end{tabular}
\label{table:TUES3}
\end{table}

From these tables we can also observe that each function in $S_{f}$ can be written in terms of a function in $S_{g}$ as shown by Table \ref{table:TUES4}, allowing that only those functions in $S_{g}$ be further utilized by this paper\footnote{Observe that the value of the quantity $i$ in Table \ref{table:TUES4} does not affect the output of the function.}.

\begin{table}[ht]
\centering
\caption{Selected relations showing connection of $S_{f}$ to $S_{g}$}
\begin{tabular}{ c c }
  $f_{1}(x) = g_{1}(x,i)$ & $f_{2}(x) = g_{4}(x,i)$ \\
  $f_{3}(x) = g_{13}(x,i)$ & $f_{4}(x) = g_{16}(x,i)$ \\
\end{tabular}
\label{table:TUES4}
\end{table}

Then the set $S_{g}$ may be used instead of the set $S_{f}$ and we select Theorem \ref{theorem:WED1} to utilize $XNOR(x,0)$ instead of $NOT(x)$, to thereby use AND, OR, and XNOR further in the paper. We show that XNOR is binary equivalent to NOT in the proof of Theorem \ref{theorem:WED1}.

\begin{theorem}
\llabel{theorem:WED1}
NOT(x) = XNOR(x,0)
\end{theorem}

\begin{proof}
Given x, then $f_{3}(x) = g_{10}(x,0)$ with:
\begin{enumerate}
  \item NOT(0) = XNOR(0,0) = 1
  \item NOT(1) = XNOR(1,0) = 0
\end{enumerate}
\label{proof:WED1}
\end{proof}

Then we can show that the extra parameter on XNOR embodies a previously constant value in the definition of logic table from \cite{CitationTruthTableDefinition}. We accomplish that through an analysis of world values, logic values, logic questions, logic tables and their expansion to a Disjunctive Normal Form (DNF) equation which computes the result designed by the logic table using the methods from \cite{CitationProcessToGenerateDNF} and \cite{CitationDefinitionOfDNF}.

During that analysis a series of examples are shown using conventional methods on discrete Boolean logical quantities (e.g. logic values from the set \{ 0, 1 \}), and an extension to a non-discrete continuous logic space on the range [0, 1] is proposed, its methods utilizing logic tables are shown, and several examples and code samples are provided.

It is then shown that by extension, a pair of [0, 1] contemplations can achieve a full [-1, 1] contemplation and hence, achieve contemplation on the arbitrary range [a, b] and by further extension, through applying a further [0, 1] set of half-space vectors, yield contemplations at a logical level in terms of multidimensional arbitrary data quantities of (x, y), (x, y, z), (r, g, b), and etc.

\subsection{Definitions of terms}

\begin{definition} A ``World Value'', labeled $W$, is an external stimulae, potentially quantized (e.g. in numeric form), representing some specific aspect of qualia as in \cite{CitationOfQualia} (e.g. hot, cold, far, near, etc.) in regard to a perception of a system.
\label{definition:WORLDVALUE}
\end{definition}

\begin{definition} A ``Logic Value'', labeled $Z$, is one of:
\begin{enumerate}
  \item A value from the set \{0, 1\}, e.g. a ``Boolean'' value.
  \item A value from an arbitrary set $S = \{ s_{0}, s_{1}, ... s_{n}\}$, e.g. a discrete or ``crisp'' value.
  \item A value from the range [0,1], e.g. a continuous or ``fuzzy'' value.
  \item A value from the arbitrary range [a, b] or dimension (x, y, ...), e.g. a ``composite'' value, which is the subject of a further section in the paper.
  \item A special value UNKNOWN, e.g. provision for an ``unknown'' value.
\end{enumerate}
\label{definition:LOGICVALUE}
\end{definition}

\begin{definition} A ``fuzzification'' after \cite{CitationOfFuzzyControlSystems} is a mapping $W \sr Z$, e.g. a ``fuzzification function'', and may be a mathematical operation that performs the mapping from a world value to a logic value when $W$ is represented in numeric form.
\label{definition:FUZZIFICATION}
\end{definition}

\begin{definition} A ``Logic Question'' is one of:
\begin{enumerate}
  \item A question answerable by a value from the set \{0, 1\}, e.g. a ``Boolean'' question.
  \item A question answerable from an arbitrary set S of states, e.g. a discrete or ``crisp'' question.
  \item A question answerable by a value from the range [0,1], e.g. a continuous or ``fuzzy'' question.
  \item A question answerable by a value from the arbitrary range [a, b] or dimension (x, y, ...), e.g. a ``composite'' question, which is the subject of a further section of the paper.
\end{enumerate}
\label{definition:TUES.2.16.2}
\end{definition}

\begin{definition}A logic question may be classified according to its difficulty, such that:
\begin{enumerate}
  \item An ``atomic question'' is a logic question that is so trivially answered that its input may be readily transformed directly into the question's answer, i.e. through {\em fuzzification}.
  \item A ``moderate question'' is a logic question that is deemed to be answered through the utilization of a single logic table.
  \item A ``hard question'' is a logic question that takes upon itself the composition of a hierarchy of two or more logic tables, and is the subject of a further section of the paper.
\end{enumerate}
\label{definition:TUES.10.22.2019.1}
\end{definition}

\begin{definition} A ``contemplation'' in this paper is a {\em Fuzzy Control System} process after \cite{CitationOfFuzzyControlSystems}, as follows:
\begin{enumerate}
  \item A fuzzification of a set of world values W into a set of logic values Z (e.g. a perception).
  \item A contemplation on a set of logic values Z, e.g. an answering of a series of logic questions, in this paper using Boolean-derived methods in continuous mathematical form.
  \item A defuzzification of some resulting logic value from Z into a world value W (e.g. a motive action).
\end{enumerate}
\end{definition}

In Problem \ref{problem:TUE.2.16.1} further in the paper we apply the fuzzy control system process by gathering a collection of world values which we are interested in, and organize them into a set $W$ which we may further utilize.

We observe that world values are of any arbitrary units and form, and so we mathematically normalize them by utilizing a set of fuzzification functions to transform the world values into logic values which we may further contemplate upon.

We define logic values in Definition \ref{definition:LOGICVALUE} to describe the various types of numeric values and their interpretations as utilized in this paper.

We answer logic questions with the methods in the paper, and produce logic values. When the logic value is strictly numeric we may further utilize its value mathematically. We may transform the logic value into a world value, or otherwise act upon it, through an action known as defuzzification.

We can observe that the action of forming a logic question is an inductive step, and the act of answering a logic question via one or more conditions upon which the question depends, is an example of a deductive step.

Examples are shown further in this paper of logic tables answering logic questions.

\subsection{Logic Tables}

We define logic tables in this paper by Definition \ref{definition:SUN1}. We may think of a logic table as accepting some specific input $I = \{ i_{0}, i_{1}, i_{2}, ..., i_{i} \}$ where each element of the set $I$ is assigned some logic value. Then the items in the set $I$ are implicitly interpreted as having a subjective meaning (e.g. health, shields, distance, etc.) and so the set $I$ represents a specific situation within which the logic table is able to produce some specific output $O$ which is itself a single logic value. Like the input values, the output value also has a subjective meaning which can be interpreted as either a computation or a contemplation of logic, so the logic table is said to output the answer to a logic question.

\begin{definition} A ``Logic Table'' consists of:

\begin{enumerate}
  \item The set $I$ of inputs $\{ i_{0}, i_{1}, i_{2}, ..., i_{i} \}$ where each element of $I$ is a logic value. When values are assigned to each member of $I$, the set represents a specific situation under which the logic table will compute the answer to a configured question.
  \item The set $O$ of outputs $\{ o_{0}, o_{1}, o_{2}, ..., o_{j} \}$ where each element of $O$ is a logic value. The values in the set $O$ each describe a specific response that the table defines for some recognized input situation.
  \item The set $M$ of matrix values $\{ m_{0,0}, m_{0,1}, m_{0,2}, ..., m_{i,j} \}$ where each element of $M$ is a logic value. The values in the set $M$ describe the situations within which the corresponding outputs from $O$ are produced.
\end{enumerate}
\label{definition:SUN1}
\end{definition}

We can then observe that the definitions of Boolean logical operations in Tables \ref{table:TUES1}, \ref{table:TUES2}, and \ref{table:TUES3} are themselves logic tables, as are any mathematical definition table, or any table listing a mapping of inputs to outputs.

We then describe how a logic table can be constructed and an equation produced which computationally maps the configured inputs to the configured outputs, and we apply the process to produce a mathematical programming of a game character's logic.

\subsection{Disjunctive Normal Form (DNF)}

We can make use of the AND, OR, and NOT\footnote{We then use XNOR instead of NOT in order to compose the DNF equation in a more customizable form.} operators by applying an established technique from \cite{CitationProcessToGenerateDNF} whereby an $I \sr O$ mapping is listed in a logic table, and a formula in terms of AND, OR, and NOT is written which computes the outputs depicted in the table when their directly-related inputs are presented to the equation. The equation produced is said to be in Disjunctive Normal Form (DNF) after \cite{CitationDefinitionOfDNF}.

Then an algorithm whereby that process is most traditionally accomplished is depicted in Algorithm \ref{algorithm:WED1}, and an algorithm depicting use of XNOR is depicted in Algorithm \ref{algorithm:WED2}.

\begin{algorithm}
\caption{Produce DNF Equation from Logic Table using NOT}
\begin{algorithmic}[1]
\STATE Given a logic table \{ $I$, $O$, $M$ \}
\STATE Let $equation$ = 0
\FOR{each row $j$ in $M$}
  \IF { $O_{j} \neq 0$ }
    \STATE Let $term$ = 1
    \FOR{each input $i$ \ELEMENTOF $I$ }
        \IF { $M_{i,j}$ = 0 }
          \STATE $term$ = AND($term$, NOT($I_{i}$))
        \ELSE
          \STATE $term$ = AND($term$, $I_{i}$)
        \ENDIF
    \ENDFOR
    \STATE $equation$ = OR($equation$, $term$)
  \ENDIF
\ENDFOR
\RETURN $equation$
\end{algorithmic}
\label{algorithm:WED1}
\end{algorithm}

\begin{algorithm}
\caption{Produce DNF Equation from Logic Table using XNOR}
\begin{algorithmic}[1]
\STATE Given a logic table \{ $I$, $O$, $M$ \}
\STATE Let $equation$ = 0
\FOR{each row $j$ in $M$}
  \IF { $O_{j} \neq 0$ }
    \STATE Let $term$ = 1
    \FOR{each input $i$ \ELEMENTOF $I$ }
      \STATE $term$ = AND($term$, XNOR($I_{i}$, $M_{i,j}$))
    \ENDFOR
    \STATE $equation$ = OR($equation$, $term$)
  \ENDIF
\ENDFOR
\RETURN $equation$
\end{algorithmic}
\label{algorithm:WED2}
\end{algorithm}

An example of the equations produced by Algorithms \ref{algorithm:WED1} and \ref{algorithm:WED2} are shown through Example \ref{example:WED1}.

\begin{example}
\parbox{10cm}{Given an arbitrary logic table, such as the XOR operation:}
\llabel{example:WED1}

\begin{tabular}{ c|c|c| p{12cm} }
  \multicolumn{3}{c}{ $I_{i}$ = \(X,Y\) } & \\
  \cline{3-3}
  \multicolumn{2}{c|}{$M_{i,j}$} & $O_{j}$ & \multirow{3}{*}{
    \parbox{12cm}{Find an equation in Disjunctive Normal Form that computes the output depicted in the table, when an input matching any of those listed in $M$ is presented as $X$, $Y$ on the input.}} \\
  \cline{1-3}
  \multicolumn{1}{|c|}{0} & 0 & 0 & \\
  \cline{1-3}
  \multicolumn{1}{|c|}{0} & 1 & 1 & \\
  \cline{1-3}
  \multicolumn{1}{|c|}{1} & 0 & 1 & \\
  \cline{1-3}
  \multicolumn{1}{|c|}{1} & 1 & 0 & \\
  \cline{1-3}
\end{tabular}
\end{example}

\begin{solution}
The equation produced by Algorithm \ref{algorithm:WED1} for the table depicted in Example \ref{example:WED1} is:
\begin{center}
  (NOT(X) AND Y) OR (X AND NOT(Y))
\end{center}
\label{solution:WED1}
\end{solution}

\begin{solution}
The equation produced by Algorithm \ref{algorithm:WED2} for the table depicted in Example \ref{example:WED1} is:
\begin{center}
  (XNOR(X,0) AND XNOR(Y,1)) OR (XNOR(X,1) AND XNOR(Y,0))
\end{center}
\label{solution:WED2}
\end{solution}

We observe that the additional parameter on the XNOR operation is itself the value from the logic table, and we observe by Theorem \ref{theorem:THU1} that the equation utilizing XNOR is identical to the equation utilizing NOT.

\begin{theorem}
The equation in Solution \ref{solution:WED2} equates to the equation in Solution \ref{solution:WED1}.
\label{theorem:THU1}
\end{theorem}

\begin{proof}
  Given Equation \ref{solution:WED1}: Apply the relations from Theorem \ref{theorem:WED1} and observe that while preserving the values of the equations there exists a process to convert one into the other.
  \begin{enumerate}
    \item XNOR(Q,0) = NOT(Q)
    \item XNOR(Q,1) = Q
  \end{enumerate}
  \label{proof:THU1}
\end{proof}

We then consider the consequences of Theorem \ref{theorem:THU1} in regard to its utility in Logic Tables. In particular we find, and utilize in Definition \ref{definition:XNOR}, the additional parameter on the XNOR operator and find that it may utilize additional form of logic values, but first please observe Examples \ref{example:SUN1} and \ref{example:SUN2}.

\begin{example}
Produce a logic table $T_{1}$ which ``recognizes'' the bit sequence \{ 1, 0, 1 \}.
\llabel{example:SUN1}
\end{example}

\begin{solution}
The logic table produced for Example \ref{example:SUN1} is:
\begin{center}
  \begin{tabular}{ c c|c|c|c|c c }
    \cline{5-5}
    $I$: & \multicolumn{1}{c}{$X$} & \multicolumn{1}{c}{$Y$} & \multicolumn{1}{c|}{$Z$} & $O$ & \multirow{2}{*}{
      \parbox{16cm}{
        $I = \{ X, Y, Z \}$\newline
        $O = \{ 1 \}$\newline
        $M = \{ 1, 0, 1 \}$.
      }} \\
    \cline{2-5}
    $M$: & \multicolumn{1}{|c|}{1} & 0 & 1 & 1 \\
    \cline{2-5}
    \multicolumn{5}{c}{}
  \end{tabular}
\end{center}
The classical DNF equation by Algorithm \ref{algorithm:WED1} is:
\begin{center}
(X AND NOT(Y) AND Z)
\end{center}
and the augmented equation by Algorithm \ref{algorithm:WED2} is:
\begin{center}
(XNOR(X,1) AND XNOR(Y,0) AND XNOR(Z,1))
\end{center}
\label{solution:SUN1}
\end{solution}

\begin{example}
Produce a logic table $T_{2}$ which ``computes'' the sum of three bits.
\llabel{example:SUN2}
\end{example}

\begin{solution}
The logic table produced for Example \ref{example:SUN2} is:
\begin{center}
  \begin{tabular}{ c c|c|c|c|c|c }
    \cline{5-6}
    $I$: & \multicolumn{1}{c}{$X$} & \multicolumn{1}{c}{$Y$} & \multicolumn{1}{c|}{$Z$} & $O_{1}$ & $O_{2}$ \\
    \cline{2-6}
    \multirow{4}{*}{$M$:} & \multicolumn{1}{|c|}{0} & 0 & 1 & 0 & 1 & \multirow{3}{*}{
      \parbox{12cm}{
        $I = \{ X, Y, Z \}$\newline
        $O_{1} = \{ 0, 0, 1, 0, 1, 1, 1 \}$\newline
        $O_{2} = \{ 1, 1, 0, 1, 0, 0, 1 \}$\newline
        $M = \{\newline
        \{ 0, 0, 1 \}, \{ 0, 1, 0 \}, \{ 0, 1, 1 \}, \{ 1, 0, 0 \},\newline
        \{ 1, 0, 1 \}, \{ 1, 1, 0 \}, \{ 1, 1, 1 \}\newline
        \}$.
      }} \\
    \cline{2-6}
    & \multicolumn{1}{|c|}{0} & 1 & 0 & 0 & 1\\
    \cline{2-6}
    & \multicolumn{1}{|c|}{0} & 1 & 1 & 1 & 0\\
    \cline{2-6}
    & \multicolumn{1}{|c|}{1} & 0 & 0 & 0 & 1\\
    \cline{2-6}
    & \multicolumn{1}{|c|}{1} & 0 & 1 & 1 & 0\\
    \cline{2-6}
    & \multicolumn{1}{|c|}{1} & 1 & 0 & 1 & 0\\
    \cline{2-6}
    & \multicolumn{1}{|c|}{1} & 1 & 1 & 1 & 1\\
    \cline{2-6}
  \end{tabular}
\end{center}
The outputs $O_{1}$ and $O_{2}$ produce separate classical DNF equations by Algorithm \ref{algorithm:WED1} as:
\begin{center}
$O_{1} = $ (NOT(X) AND Y AND Z) OR (X AND NOT(Y) AND Z) OR (X AND Y AND NOT(Z)) OR (X AND Y AND Z))
\end{center}
\begin{center}
$O_{2} =$ (NOT(X) AND NOT(Y) AND Z) OR (NOT(X) AND Y AND NOT(Z)) OR (X AND NOT(Y) AND NOT(Z)) OR (X AND Y AND Z)
\end{center}
with augmented formulas given by Algorithm \ref{algorithm:WED2} as:
\begin{center}
$O_{1} =$ (XNOR(X,0) AND XNOR(Y,1) AND XNOR(Z,1)) OR (XNOR(X,1) AND XNOR(Y,0) AND XNOR(Z,1)) OR (XNOR(X,1) AND XNOR(Y,1) AND XNOR(Z,0)) OR (XNOR(X,1) AND XNOR(Y,1) AND XNOR(Z,1))
\end{center}
\begin{center}
$O_{2} =$ (XNOR(X,0) AND XNOR(Y,0) AND XNOR(Z,1)) OR (XNOR(X,0) AND XNOR(Y,1) AND XNOR(Z,0)) OR (XNOR(X,1) AND XNOR(Y,0) AND XNOR(Z,0)) OR (XNOR(X,1) AND XNOR(Y,1) AND XNOR(Z,1))
\end{center}
\label{solution:SUN1}
\end{solution}

\subsection{Extensions to continuous logic}

We find continuous mathematical definitions for classical Boolean operators remaining relatively unchanged from history,
with the definition of NOT, AND, and OR going as far back as \cite{CitationOfGeorgeBoole} and \cite{CitationOfCharlesPeirce}.

When selecting continuous mathematical functions that compute NOT, AND, OR, and XNOR, we may then seek any of those functions that produce the same \{0,1\} output values for any given \{0,1\} input values, and consider alternatives\footnote{We have found $X \oplus Y = min(|x+y|,1)$ to behave more accurately than $x+y-x\*y$.} to the definitions here as affecting the values produced between 0 and 1. We analyze and observe the results. In particular, we examine the following:

\begin{definition}

\parbox{5cm}{
  \begin{tabular}{ c|c|c }
    \cline{2-2}
    \multicolumn{1}{c|}{$X$} & $NOT$ \\
    \cline{1-2}
    \multicolumn{1}{|c|}{0} & 1 \\
    \cline{1-2}
    \multicolumn{1}{|c|}{1} & 0 \\
    \cline{1-2}
  \end{tabular}
  $= (1 - X)$
}
\parbox{6cm}{
  The {\em complement} of a logical quantity produces a depiction of its lack of contribution to a whole truth.
}
\end{definition}

\begin{definition}

\parbox{5cm}{
  \begin{tabular}{ c|c|c|c }
    \cline{3-3}
    \multicolumn{1}{c}{$X$} & \multicolumn{1}{c|}{$Y$} & $AND$ \\
    \cline{1-3}
    \multicolumn{1}{|c|}{0} & 0 & 0 \\
    \cline{1-3}
    \multicolumn{1}{|c|}{0} & 1 & 0 \\
    \cline{1-3}
    \multicolumn{1}{|c|}{1} & 0 & 0 \\
    \cline{1-3}
    \multicolumn{1}{|c|}{1} & 1 & 1 \\
    \cline{1-3}
  \end{tabular}
  $= X * Y$
}
\parbox{6cm}{
  The {\em multiplication} of two logical quantities produces a depiction of their mutual contribution to a whole truth.
}
\end{definition}

\begin{definition}

\parbox{5cm}{
  \begin{tabular}{ c|c|c|c }
    \cline{3-3}
    \multicolumn{1}{c}{$X$} & \multicolumn{1}{c|}{$Y$} & $OR$ \\
    \cline{1-3}
    \multicolumn{1}{|c|}{0} & 0 & 0 \\
    \cline{1-3}
    \multicolumn{1}{|c|}{0} & 1 & 1 \\
    \cline{1-3}
    \multicolumn{1}{|c|}{1} & 0 & 1 \\
    \cline{1-3}
    \multicolumn{1}{|c|}{1} & 1 & 1 \\
    \cline{1-3}
    \multicolumn{4}{c}{$= X \oplus Y = x+y$ : [0,1] } \\
  \end{tabular}
  }
  \parbox{6cm}{
    The {\em capped addition} of two logical quantities produces a depiction of an independent contribution of each quantity toward a whole truth.
  }
  \label{definition:BOOLEAN_OR}
\end{definition}

\begin{definition}

\parbox{5cm}{
  \begin{tabular}{ c|c|c|c }
    \cline{3-3}
    \multicolumn{1}{c}{$X$} & \multicolumn{1}{c|}{$Y$} & $XNOR$ \\
    \cline{1-3}
    \multicolumn{1}{|c|}{0} & 0 & 1 \\
    \cline{1-3}
    \multicolumn{1}{|c|}{0} & 1 & 0 \\
    \cline{1-3}
    \multicolumn{1}{|c|}{1} & 0 & 0 \\
    \cline{1-3}
    \multicolumn{1}{|c|}{1} & 1 & 1 \\
    \cline{1-3}
    \multicolumn{4}{c}{$= 1-|X - Y| = EQ(X,Y)$, or} \\
  \end{tabular}
  }
  \parbox{6cm}{
    The {\em XNOR} of two logical quantities produces a depiction of the extent as to which two logical quantities agree toward the value of a truth.
  }
  \\
  \hspace*{3cm}$(1-|X - Y|)^{i}$ \\
  \hspace*{3cm}$= EQ(X, Y, i)$ \\
  \label{definition:XNOR}
\end{definition}

Where the operator $\oplus$ denoted here is {\em addition} restricted to the range [0,1], the XNOR function is labeled $EQ$ for clarity\footnote{Observe that the Boolean XNOR operation can indeed be described as an {\em equals} operator}, $i$ provides inference control and is the subject of a further section of the paper, and the operator $|x|$ denotes the {\em absolute value} operator.

We can then rewrite Solution \ref{solution:WED2} to produce the equation shown in Solution \ref{solution:WED3} which we call the {\em continuous form} of its logical expression.

\begin{solution}
The equation from Solution \ref{solution:WED2} rewritten in ``continuous form'' is:
\begin{center}
  $(EQ(X,0) * EQ(Y,1)) \oplus (EQ(X,1) * EQ(Y,0))$
\end{center}
\label{solution:WED3}
\end{solution}

When graphed the equation from Solution \ref{solution:WED3} produces the surface depicted in Figure \ref{figure:WED1} and matches the expected output of its definition from Example \ref{example:WED1}, which is the XOR function.

\begin{figure}
\centering
    \includegraphics[totalheight=8cm]{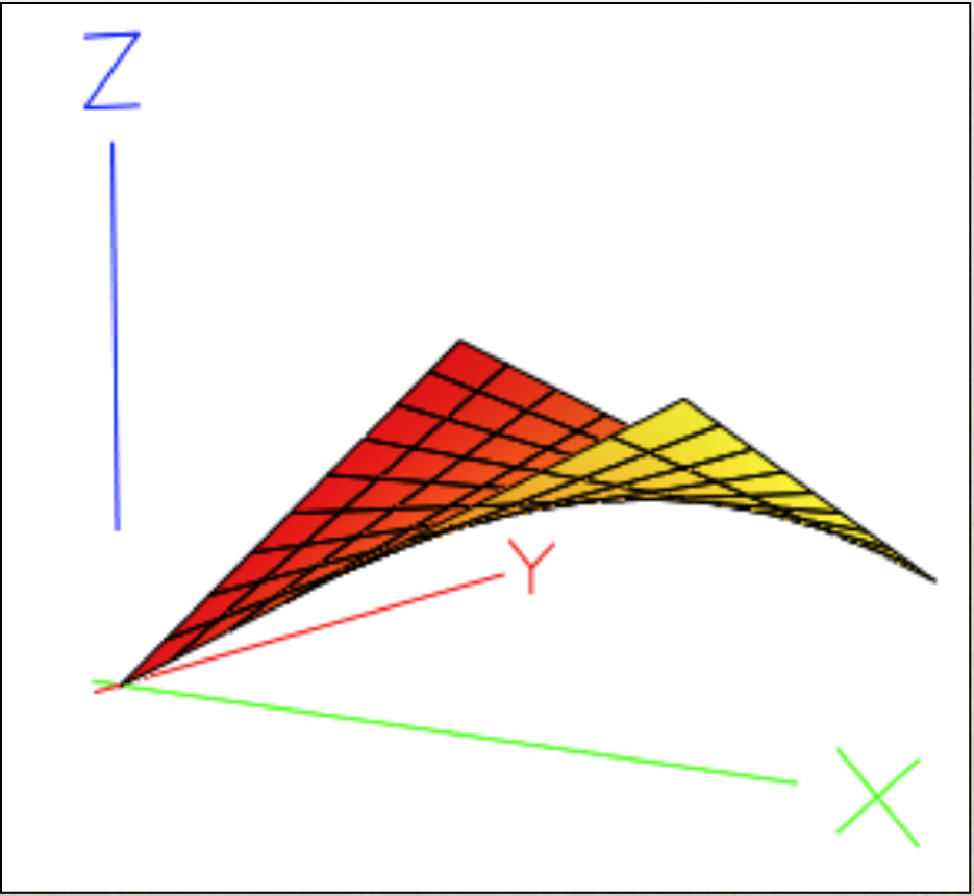}
    \caption{Graphed output of the equation from Solution \ref{solution:WED3}}
    \label{figure:WED1}
\end{figure}

We next apply a trivial arithmetic operation to cause each term of the equations emitted by Algorithm \ref{algorithm:WED2} to support arbitrary row outputs which can be configured to be in the range [0,1]. We accomplish that by applying a multiplier to each term, thereby transforming whether the term matches its configured values, into what we want the term to output given that it matches its configured values, producing the listing in Algorithm \ref{algorithm:WED3}.

\subsection{Provision for continuous outputs}

\begin{theorem} On each term of the equation in Solution \ref{solution:WED3}, where we have a value on the range [0,1], we may apply a multiplier Z, producing terms of the form Z*(EQ(X,Y)*...). The resulting transformation is from a logical value on the range [0,1] to a value that becomes Z to the extent that X==Y, or \{ Z,Z,... \} to the extent that \{ X,X,... \} == \{ Y,Y,... \}.
\label{theorem:WED2}
\end{theorem}

\begin{proof}
If EQ(X,Y) is whether X==Y by Definition \ref{definition:XNOR}, then Z*EQ(X,Y) is Z to the extent that X==Y.
\end{proof}

\begin{example}
\label{example:CONTINUOUS_FORM}
\end{example}

We apply a multiplier to each term and observe that the output interpolates to express the values depicted in the multipliers.
\begin{enumerate}
 \item $O_{0}*(EQ(X,0) * EQ(Y,1)) \oplus O_{1}*(EQ(X,1) * EQ(Y,0))$
\end{enumerate}

We reproduce the formula as follows, and observe that the nature of the formula is multidimensional interpolation.

Given the set $I = \{ i_{0}, i_{1}, i_{2}, ..., i_{i} \}$ of arbitrary input values on the range [0,1].

Given the set $M = \{ M_{0,0}, M_{0,1}, ..., M_{i,j} \}$ of arbitrary matrix values on the range [0,1].

Given the set $O = \{ O_{0}, O_{1}, ..., O_{j} \}$ of arbitrary output values on the range [0,1].

We can form a table as follows:

\begin{tabular}{ c|c|c|c|c| p{12cm} }
  \multicolumn{5}{c}{ $I_{i} = \{ i_{0}, i_{1}, i_{2}, ..., i_{i} \}$ } & \\
  \cline{5-5}
  \multicolumn{4}{c|}{$M_{i,j}$} & $O_{j}$ & \multirow{3}{*}{
    \parbox{12cm}{Where the input values $I_{i}$ are applied upon the matrix values $M_{i,j}$ via the process depicted in Algorithm \ref{algorithm:WED3}, to produce an interpolation of the output as configured in the set $O$.}} \\
  \cline{1-5}
  \multicolumn{1}{|c|}{$m_{0,0}$} & $m_{1,0}$ & ... & $m_{i,0}$ & $o_{0}$ & \\
  \cline{1-5}
  \multicolumn{1}{|c|}{$m_{0,1}$} & $m_{1,1}$ & ... & $m_{i,0}$ & $o_{1}$ & \\
  \cline{1-5}
  \multicolumn{1}{|c|}{$m_{0,2}$} & $m_{1,2}$ & ... & $m_{i,0}$ & $o_{2}$ & \\
  \cline{1-5}
  \multicolumn{1}{|c|}{$...$} & $...$ & $...$ & $...$ & ... & \\
  \cline{1-5}
  \multicolumn{1}{|c|}{$m_{0,j}$} & $m_{1,j}$ & ... & $m_{i,j}$ & $o_{j}$ & \\
  \cline{1-5}
\end{tabular}

Producing an equation of the form:
\begin{center}
  $E = o_{0}*(EQ(i_{0},m_{0,0}) * EQ(i_{1},m_{1,0}) * ... * EQ(i_{i},m_{i,0})) \oplus ... \oplus o_{j}*(EQ(i_{0},m_{0,j}) * EQ(i_{1},m_{1,j}) * ... * EQ(i_{i},m_{i,j}))$
\end{center}

Where Algorithm \ref{algorithm:WED3} extends the prior algorithms to account for the rescaling of the terms of the DNF equation, and to produce the full formula as shown in Example \ref{example:CONTINUOUS_FORM}.

\begin{algorithm}
\caption{Production of a ``Continuous Form'' Equation from Logic Table using XNOR (labeled ``EQ'')}
\begin{algorithmic}[1]
\STATE Given a logic table \{ $I$, $O$, $M$ \}
\STATE Let $equation$ = 0.0
\FOR{each row $j$ in $M$}
    \IF { $O_{j} \neq 0.0$ }
      \STATE Let $term$ = $O_{j}$
      \FOR{each input $i$ \ELEMENTOF $I$ }
        \STATE $term$ = AND($term$, EQ($i_{i}$, $m_{i,j}$))
      \ENDFOR
      \STATE $equation$ = OR($equation$, $term$)
    \ENDIF
\ENDFOR
\RETURN $equation$
\end{algorithmic}
\label{algorithm:WED3}
\end{algorithm}

\begin{theorem}
\llabel{theorem:WED.1.6.5.21}
Algorithm \ref{algorithm:WED3} produces a continuous equation which performs multidimensional interpolation.\end{theorem}

\begin{proof}
  We seek to demonstrate multidimensional interpolation in a proof by induction. We start by generating and expanding a logic table of a single variable and output as follows...
  \begin{center}
  $E = O_{0} * (XNOR(X, M_{0,0})) = O_{0} * ((1.0 - |X - M_{0,0}|))$
  \end{center}
  ...observing that the output of the term will match the value configured in $O_{0}$ when the value of the input $X$ matches the value of the table value $M_{0,0}$. We observe also that as $X$ deviates from $M_{0,0}$, that less of the output value $O_{0}$ contributes to the output of the equation.
  We add an additional variable to the equation, still with only a single row, and we produce the following...
  \begin{center}
  $E = O_{0} * (XNOR(X, M_{0,0}) * XNOR(Y, M_{1,0})) = O_{0} * ((1.0 - |X - M_{0,0}|)*(1.0 - |Y - M_{1,0}|))$
  \end{center}
  ...observing that the quantity $(X,Y)$ must match $(M_{0,0}, M_{1,0})$ for the equation to produce the value $O_{0}$ on its output. It follows that for each additional variable we input into the equation that an additional quantity must be matched in the currently single row of the table in order for the configured output to be emitted.
  We next add an additional row to the logic table, producing an OR condition on the recognition of that row...
  \begin{center}
  $E = O_{0} * (XNOR(X, M_{0,0}) * XNOR(Y, M_{1,0})) \oplus O_{1} * (XNOR(X, M_{0,1}) * XNOR(Y, M_{1,1})) = O_{0} * ((1.0 - |X - M_{0,0}|)*(1.0 - |Y - M_{1,0}|)) \oplus O_{1} * (1.0 - |X - M_{0,1}|) * (1.0 - |Y - M_{1,1}|))$
  \end{center}
  ...and we observe that the output of the equation approaches $O_{0}$ to the extent that $(X,Y)$ approaches $(M_{0},M_{1})$, while also approaching the value of $O_{1}$ to the extent that the quantity $(X,Y)$ approaches $(M_{0,1},M_{1,1})$, resulting in an interpolation between the values depicted in the output array $O_{j}$.
  \label{proof:WED.1.6.5.21}
\end{proof}

We next provision the formula with support for unknowns, and crisp state values, and finally show examples.

\subsection{Provision for unknowns}

We have observed what we might call {\em Boolean Values}, e.g. \{ x : x \ELEMENTOF \{ 0, 1 \} \}, and {\em Continuous Values}, e.g. \{ z : z \ELEMENTOF [0,1] \}, and we next provision with support for additional data as defined as follows.

In particular we produce Algorithm \ref{algorithm:WED4} to allow for the special value UNKNOWN, enabling the machine to be configured so as to not necessitate a full connection of all its inputs to outputs.

\begin{algorithm}
\caption{Production of a ``Continuous Form'' Equation from Logic Table using EQ w/ Provision for Unknowns}
\begin{algorithmic}[1]
\STATE Given a logic table \{ $I$, $O$, $M$ \}
\STATE Let $equation$ = 0.0
\FOR{each row $j$ in $M$}
    \IF { $O_{j} \neq 0.0$ }
      \STATE Let $term$ = $O_{j}$
      \FOR{each input $i$ \ELEMENTOF $I$ }
        \IF{$m_{i,j}$ is not $UNKNOWN$}
          \STATE $term$ = AND($term$, EQ($i_{i}$, $m_{i,j}$))
        \ENDIF
      \ENDFOR
      \STATE $equation$ = OR($equation$, $term$)
    \ENDIF
\ENDFOR
\RETURN $equation$
\end{algorithmic}
\label{algorithm:WED4}
\end{algorithm}

\subsection{Provision for state machines}

We provision the EQ function in Algorithm \ref{algorithm:WED5} with support for crisp {\em State Values}, which we can describe simply as those whole numbered values, e.g. $\{ n : n$ \ELEMENTOF $\nat \}$. We then can configure the logic table to recognize state values and also emit them.

\begin{algorithm}
\caption{EQ w/ support for ``state values''}
\begin{algorithmic}[1]
  \STATE Given $X$, $Y$
  \IF{X and Y are state values}
    \RETURN $(X == Y)$
  \ELSE
    \RETURN $1.0 - |X - Y|$
  \ENDIF
\end{algorithmic}
\label{algorithm:WED5}
\end{algorithm}

\subsection{Provision for inference control}

The proposed formula for the XNOR operation achieves a continuous logic table that performs inference on the logic values produced by the formula of the table. For example, if a logic table is configured to emit a value of 1.0 when a value of 0.5 is presented on its input, the inference performed by the continuous XNOR operator will cause values surrounding 0.5 such as 0.4 or 0.6 to produce values near 1.0 such as 0.9.

The level of inference can be controlled by using a modifier supplied in the extended description of the XNOR operator in Definition \ref{definition:XNOR}.

Accordingly the EQ function is augmented in Algorithm \ref{algorithm:TUE.9.55.NOV.5.2019} to support inference control modifiers.

\begin{algorithm}
\caption{EQ w/ support for ``inference control''}
\begin{algorithmic}[1]
  \STATE Given $X$, $Y$, $i$
  \IF{X and Y are state values}
    \RETURN $(X == Y)$
  \ELSE
    \RETURN $(1.0 - |X - Y|)^{i}$
  \ENDIF
\end{algorithmic}
\label{algorithm:TUE.9.55.NOV.5.2019}
\end{algorithm}

\subsection{Examples}

\begin{problem}
Show the construction of a machine which plays a primitive game of {\em Soccer}, while being programmed not in the traditional sense, but programmed {\em probabilistically} with equations of the form depicted prior.
\label{problem:TUE.2.16.1}
\end{problem}

We approach the answer to Problem \ref{problem:TUE.2.16.1} in Solution \ref{solution:MON.1.11} by constructing a series of sets of arbitrary information that will be required to produce the sets needed for the logic tables. In particular we are depicting a {\em Fuzzy Control System} process as in \cite{CitationOfFuzzyControlSystems}.

We will let $W = \{ w_{1}, w_{2}, w_{3}, ..., w_{w} \}$ be the set of arbitrary ``real world'' values that the character may perceive.
These are quantities which may be considered to be of any form or composition. The set W for this example is then given in Table \ref{table:MON.1.34}.

\begin{table}[ht]
\centering
\caption{World Variables for Solution \ref{solution:MON.1.11}}
\begin{tabular}{ |c|l| }
\cline{1-2}
  World Variable & \multicolumn{1}{c|}{Definition} \\
  \cline{1-2}
  $w_{0}$ & The $P=(x,y)$ position of the robot. \\
  \cline{1-2}
  $w_{1}$ & The $F=<x,y>$ normalized vector in the forward facing direction of the robot. \\
  \cline{1-2}
  $w_{2}$ & The $R=<x,y>$ normalized direction to the robot's right\footnote{In two dimensions this can be inferred from the forward vector by $R_{x}=-F_{y}, R_{y}=F_{x}$.}. \\
  \cline{1-2}
  $w_{3}$ & The $Q=(x,y)$ position of the robot's current target\footnote{e.g. a soccer ball or goal.} \\
  \cline{1-2}
  $w_{4}$ & The $D$ distance to the robot's current target. \\
  \cline{1-2}
  $w_{5}$ & The $B=(x,y)$ position of the soccer ball. \\
  \cline{1-2}
  $w_{6}$ & The $G=(x,y)$ position of the goal. \\
  \cline{1-2}
  $w_{7}$ & An H \ELEMENTOF \{0,1\} indicator of whether the ball is ``held''. \\
  \cline{1-2}
  $w_{8}$ & The $V = (Q-P)/|Q-P|$ normalized vector from the robot to the current target. \\
  \cline{1-2}
  $w_{9}$ & The $F \cdotp V$ dot product of the robot's forward vector with the vector to the target. \\
  \cline{1-2}
  $w_{10}$ & The $R \cdotp V$ dot product of the robot's right vector with the vector to the target. \\
  \cline{1-2}
\end{tabular}
\label{table:MON.1.34}
\end{table}

Our next step is to apply a normalization of the values from W into the character's [0,1] logic space.
This step is called the {\em fuzzification} step after \cite{CitationOfFuzzyControlSystems}.

Then let the set $S = \{ s_{1}, s_{2}, s_{3}, ..., s_{s} \}$ be the arbitrary set of normalized sensor values. These are continuous logic values on the range [0,1] and represent continuous Boolean logic (to distinguish it from discrete Boolean logic). We infer this set of values from the set W as shown in Table \ref{table:MON.1.38}.

\begin{table}[ht]
\centering
\caption{Sensor Variables for Solution \ref{solution:MON.1.11}}
\begin{tabular}{ |c|l|l| }
\cline{1-3}
  Sensor Variable & \multicolumn{1}{c}{Definition} & \multicolumn{1}{|c|}{Calculation} \\
  \cline{1-3}
  $s_{0}$ & Is the target to my front? & clamp($w_{9}, 0, 1$) \\
  \cline{1-3}
  $s_{1}$ & Is the target to my back? & clamp($-1*w_{9}, 0, 1$) \\
  \cline{1-3}
  $s_{2}$ & Is the target near? & map($w_{4}, 0, 400, 1, 0$) \\
  \cline{1-3}
  $s_{3}$ & Is the target to my right? & clamp($w_{10}, 0, 1$) \\
  \cline{1-3}
  $s_{4}$ & Is the target to my left? & clamp($-1*w_{10}, 0, 1$) \\
  \cline{1-3}
  $s_{5}$ & Do I have the ball? & $w_{7}$ \\
  \cline{1-3}
\end{tabular}
\label{table:MON.1.38}
\end{table}

\begin{table}[ht]
\centering
\caption{Miscellaneous functions}
\begin{tabular}{ |c|l| }
\cline{1-2}
  Function & \multicolumn{1}{c|}{Definition} \\
  \cline{1-2}
  clamp($x$, $min$, $max$) & if $x<min$ return $min$ else if $x>max$ return $max$ else return $x$. \\
  \cline{1-2}
  map($x$, $min1$, $max1$, $min2$, $max2$) & return $min2 + ((x - min1) / (max1 - min1)) * (max2 - min2)$. \\
  \cline{1-2}
\end{tabular}
\label{table:MON.6.50}
\end{table}

Where the miscellaneous mathematical functions such as ``clamp'' and ``map'' are listed in Table \ref{table:MON.6.50}, and are being used to perform the conversion $W \sr S$, and are called {\em fuzzification functions}. We observe that there may exist many such functions.

We may then define a series of continuous logic tables to describe each possible output behavior of the character as given in Solution \ref{solution:MON.1.11}.

\begin{solution} Logic tables for Problem \ref{problem:TUE.2.16.1}.
\label{solution:MON.1.11}

\begin{tabular}{ c|c|c }
  \multicolumn{2}{c}{ Should I drive forward? } & \\
  \multicolumn{2}{c}{ $I_{i} = \{ s_{0} \}$ } & \\
  \cline{2-2}
  \multicolumn{1}{c|}{$M_{i,j}$} & $O_{j}$ & \multirow{3}{*}{
    \parbox{12cm}{= 1.0 * EQ($s_{0}$, 1.0)}} \\
  \cline{1-2}
  \multicolumn{1}{|c|}{$1.0$} & $1.0$ & \\
  \cline{1-2}
\end{tabular}\begin{center}\end{center}

\begin{tabular}{ c|c|c|c|c }
  \multicolumn{4}{c}{ Should I throw the ball? } & \\
  \multicolumn{4}{c}{ $I_{i} = \{ s_{0}, s_{2}, s_{5} \}$ } & \\
  \cline{4-4}
  \multicolumn{3}{c|}{$M_{i,j}$} & $O_{j}$ & \multirow{3}{*}{
    \parbox{12cm}{= 1.0 * EQ($s_{0}$, 1.0) * EQ($s_{2}$, 0.75) * EQ($s_{5}$, 1.0)}} \\
  \cline{1-4}
  \multicolumn{1}{|c|}{$1.0$} & $0.75$ & $1.0$ & $1.0$ & \\
  \cline{1-4}
\end{tabular}\begin{center}\end{center}

\begin{tabular}{ c|c|c|c }
  \multicolumn{3}{c}{ Should I turn to the right? } & \\
  \multicolumn{3}{c}{ $I_{i} = \{ s_{1}, s_{3} \}$ } & \\
  \cline{3-3}
  \multicolumn{2}{c|}{$M_{i,j}$} & $O_{j}$ & \multirow{3}{*}{
    \parbox{12cm}{= 1.0 * EQ($s_{3}$, 1.0) $\oplus$ 1.0 * EQ($s_{1}$, 1.0) * EQ($s_{3}$, 1.0) }} \\
  \cline{1-3}
  \multicolumn{1}{|c|}{UNK} & $1.0$ & $1.0$ & \\
  \cline{1-3}
  \multicolumn{1}{|c|}{$1.0$} & $1.0$ & $1.0$ & \\
  \cline{1-3}
\end{tabular}\begin{center}\end{center}

\begin{tabular}{ c|c|c|c }
  \multicolumn{3}{c}{ Should I turn to the left? } & \\
  \multicolumn{3}{c}{ $I_{i} = \{ s_{1}, s_{4} \}$ } & \\
  \cline{3-3}
  \multicolumn{2}{c|}{$M_{i,j}$} & $O_{j}$ & \multirow{3}{*}{
    \parbox{12cm}{= 1.0 * EQ($s_{4}$, 1.0) $\oplus$ 1.0 * EQ($s_{1}$, 1.0) * EQ($s_{4}$, 1.0) }} \\
  \cline{1-3}
  \multicolumn{1}{|c|}{UNK} & $1.0$ & $1.0$ & \\
  \cline{1-3}
  \multicolumn{1}{|c|}{$1.0$} & $1.0$ & $1.0$ & \\
  \cline{1-3}
\end{tabular}\begin{center}\end{center}

\begin{tabular}{ c|c|c }
  \multicolumn{2}{c}{ Where should I target (X,Y)? } & \\
  \multicolumn{2}{c}{ $I_{i} = \{ s_{5} \}$ } & \\
  \cline{2-2}
  \multicolumn{1}{c|}{$M_{i,j}$} & $O_{j}$ & \multirow{3}{*}{
    \parbox{12cm}{= $w_{6}$ * EQ($s_{5}$, 1.0) $\oplus$ $w_{5}$ * EQ($s_{5}$, 0.0) }} \\
  \cline{1-2}
  \multicolumn{1}{|c|}{1.0} & $w_{6}$ & \\
  \cline{1-2}
  \multicolumn{1}{|c|}{0.0} & $w_{5}$ & \\
  \cline{1-2}
\end{tabular}
\end{solution}

We then, at periodic intervals update the sets W and S, and compute a set Z to contain the answer to the questions depicted in the logic tables, which we may then directly ``defuzzify'' into motive actions. In particular, we may multiply the logic value of whether to drive forward with a value representing how fast to travel at maximum, and apply it to the character's position causing the character to travel forward. We may also defuzzify the logical quantities about turning right or left by multiplying each of them by a quantity depicting how much to turn at maximum, and then turning in whichever direction represents a greater recommendation value, and we may defuzzify the logical quantity about throwing the ball, by evaluating whether it exceeds a threshold value (such as 0.90). We pick up the ball when the character is sufficiently close to it, setting the appropriate world variable, $w_{7}$, to 1, and release the ball when the character decides to throw it by setting the variable to 0, and we apply motion to the ball depicting it thrown. We can then reset the position of the ball, or indeed, keep it at its arbitrary place on the field, and we observe that the character plays a primitive game of soccer, picking up the ball, carrying it to, and throwing it in the goal, and will continue doing so perpetually as long as the character runs.

\subsection{Application to arbitrary range, precision, and dimensionality}

The logic tables described thus far in this paper embody logic on the range [0,1], but can be utilized in the following fashions to operate on the range [-1,1], including the arbitrary range [a,b], and also multidimensionally, e.g. logic values (x,y), (x,y,z), etc. i.e. ([a,b],[c,d]), ([a,b],[c,d],[e,f]), and etc. thereby allowing the system of logic to contemplate on notions of arbitrary data and precision, and also to think in terms of color and spatial dimensions.

First, to provide an input to the logic system on the arbitrary range [0,b], the input value is downscaled into the internal range [0,1] whilst being input into the system by performing a division of its ranged value $b$, and taking its absolute value. In that way the arbitrary range [0,b] is supported while internally the logic system operates on the range [0,1]. Given that, an input value $X$ of range [-1,1] can be constructed out of two separate internal input values $X_{a}$ [0,1], and $X_{b}$ [0,1], such that when $X$ is between -1 and 0 the absolute of its value is supplied to the $X_{a}$ input, and when $X$ is between 0 and 1 it is supplied to the $X_{b}$ input thereby allowing the logic system to internally operate on the range [0,1] while supporting an external impression that it is operating on the range [-1,1]. The arbitrary range [a,b] can then be created, trivially if [a,b] are both of the same sign by shifting and scaling the numeric system onto the range [0, 1] through a trivial addition and multiplication, or by treatment of two separate inputs on the range [0,+a] and [0,+b] respectively, performing a scaling and absoluting of the input values into the appropriate [0,1] range prior to inputing into the logic table. Given that [a, b] are different signs, when a particular input is on the range [0, a] it is transformed and applied to the [0, +a] input, while the [0, +b] input remains zero, and vice-versa. Multidimensional inputs are then possible by adding even further inputs to accomodate the separate dimensions of the input, allowing the production of logic questions which operate on input data that are conceptually of arbitrary data quantities, and that are conceptually multidimensional, e.g. spatial values (x,y,z) and etc. and color values (r,g,b), thereby allowing production of contemplations in terms of arbitrary data quantities, colors, and spatial constructs.

Second, providing output on an arbitrary range [0,b] can be accomplished by scaling the [0,1] output of the logic system by multiplying its output by its range value $b$. In that way the arbitrary range [0,b] is supported while internally the logic system operates on the range [0,1]. Given that, an output value $X$ of range [-1,1] can be constructed out of two separate logic tables $X_{a}$ [0,1], and $X_{b}$ [0,1], such that the logic table $X_{a}$ decides whether the effective output value occupies a range from [-1,0] (the value of the logic table being negated), and the logic table $X_{b}$ decides whether the effective output value occupies a range from [0,1]. Though there are two separate values determined by two separate tables internally, the table with the greater recommendation enjoys its value to be taken as the single conceptual output. An example of this can be observed in the robot soccer example of Problem \ref{problem:TUE.2.16.1} where the robot effectively decides whether to turn left or right, essentially answering the logic question "How should I turn?" on a conceptual range of [-1,1] by resolving both directions separately and taking the greater value of the two in the production of its turn. Via further value scaling the logic question may then be produced on the arbitrary range [a,b], and through provision of additional tables the output value may then be conceptually multidimensional e.g. (x,y) as ([a,b],[c,d]), or (x,y,z) as ([a,b],[c,d],[e,f]), and etc. allowing the logical contemplations to produce answers to logic questions as points of multidimensional space, and thereby answering logic questions in terms of arbitrary data quantities, colors, and spatial constructs.

\subsection{Application to ``hard'' questions}
At this point the reader might wonder why the initial logic questions in Problem \ref{problem:TUE.2.16.1} are answered with the {\em fuzzification} functions ``map'' and ``clamp'' rather than with logic tables, and the answer is that they are deemed {\em atomic questions}, as defined in Definition \ref{definition:TUES.10.22.2019.1}, and so are such readily answerable that they do not require the methods of logic tables described in the paper to produce their answers. Contrarily, the questions answered further in Problem \ref{problem:TUE.2.16.1} are in and of themselves not so clear in how to produce their answers, but are considered ``moderate questions'' in that their being answered is indicated in utilizing a single logic table. Then, there exists a category of logic questions that we may call ``hard questions'' that are so complicated that we cannot immediately imagine them requiring a single logic table, though they potentially could, but we imagine their answers as requiring two or more logic tables, not merely of a composite result, but of a hierarchy of logic tables where each contributes in its own way toward the cumulative answer.

As it happens, there exists an interesting technique employed by the author whereby a hierarchy of logic tables may be constructed so as to produce the answer to such a {\em hard question}. We begin with the hard question unconfigured, which we consider could be any question that is indicated to produce an answer that is deemed a logic value, and we next recursively split the hard question into further questions that the hard question depends upon, and we observe that the dependent questions are of a lower or equal difficulty than the hard question itself (for if they were more difficult, the original question would be deemed itself even more difficult further by cumulative laws of composition), such that they may be atomic questions, moderate questions, or hard questions of their own. We then recursively split each of the remaining questions into further questions, and so on, until we encounter a total set of atomic questions whose answers feed as inputs into the moderate or hard questions which they were split from, whose own answers are then fed into the questions they were split from, and so on, until we have a hierarchy depicting the answer to the original hard question, or else we have encountered either an infinite loop, or an inability to split a question, in which case the original question is not computable, at least at the present time with the available information at hand (something that should never occur in answerable questions).

\subsection{Application to machine learning}

Machine learning can be accomplished on these logic constructs via the dynamic insertion of rows
into the logic table matrices. In particular an interesting form of machine learning exists whereby
a human operates the artificially intelligent agent and the sensor values as they are at the time of the
human's motor inputs are inserted into the appropriate logic table matrices.

For example, and as for the example in this paper, a human might press forward, left, right, or a button indicating whether to throw the ball, and the soccer robot's sensor values pertaining to whether the goal is to the right or left, forward, ball held, and etc. are inserted into the appropriate logic table corresponding to the human's particular input, e.g. the logic tables that indicate whether to drive forward, left, right, throw the ball, and etc.

Additional sensors e.g. the placement of opponents on the field can be pruned from their insertion according to contradictions on input, for example if a character is now right when they were once left they can be recognized as not relevent to the input because their value is so varied across training examples, thereby serving refinement in hope of a reduced most set of sensor inputs to be considered at any given time.

Additionally the sensor inputs may be quantized onto strict increments, e.g. 0.1, 0.2, etc. so that the human may retain the input button down and not overwhelm the matrix with a multitude of sensor values. In this manner the human may thereby passively and simply play a game of robot soccer and observe that the machine has in that fashion ``learned'' to play the game of soccer in the same fashion, including passing the ball, defensive and offensive participation, and indeed every behavior and action committed by the human subject to the sufficient set of sensors available to the intelligent character. This method of learning has been titled ``observation training'' by the author, and familiarly titled ``Monkey See Monkey Do''. The example as just described has been implemented and verified by the author, but it can also be imagined that an airplane with sufficient sensors can in the same fashion learn to land itself by being landed by a human multiple times provided it is sufficiently tuned in regard to which sensors are necessary to pay attention to at any given time. The act of reducing the sensors is a particular problem but has had some success in refining via identifying contradictions on input as described, e.g. the human presses left in a situation where they pressed right previously, or a particular sensor's range of values exceeds a certain range and so is considered to not be meaningful enough. This problem appears to compound as the number of available sensors increases and could be the subject of further analysis. Nevertheless the author believes that the foundation described as such in this paper provides sufficient support for a theory of contemplation.

\subsection{Change History}
v1: 01/01/2016: Original paper. Solid and reliable in its primary description.\\
v2: 10/27/2018: Added clarifying parenthesisation around $(Q - P)/|Q - P|$. \\
v3: 09/26/2019: Application to arbitrary range, precision, and dimensionality. \\
v4: 10/20/2019: Application to machine learning. \\
v5: 10/22/2019: Application to hard questions, miscellaneous edits, and miscellaneous clarification of definitions w.r.t. added sections. \\
v6: 10/30/2019: Type-O "seprate" $\rightarrow$ "seperate". Added comments on inference control modifiers. \\
v7: 11/05/2019: Added reference and erroneous refinement of inference control modifiers. \\
v8: 11/07/2019: Return inference control modifier to practice. Added change history. \\

\subsection{Conclusions}

This paper summarizes a process whereby logical machines may be constructed and offers a proposed theory of contemplation for further growth and analysis.

{\em Acknowledgements:} The views and conclusions contained herein are those of the author and should not be interpreted as necessarily representing the official policies or endorsements, either expressed or implied, of the Art Institute of Portland.


\def\cprime{$'$}

\end{document}